\documentclass[twoside]{article}

\usepackage[accepted]{aistats2022}
\usepackage{xargs}
\usepackage{papercommands}
\newcommand{\fft}{Sine fitting }
\renewcommandx{\mod}[2][2=\frac{N}{2}]{#1 \ \ensuremath{\mathrm{mod}\  #2}}

\newcommand{\floor}[1]{\left\lfloor #1  \right\rfloor}
\usepackage{url}
\usepackage[linesnumbered,lined,ruled,noend]{algorithm2e}

\usepackage{amsmath,amsthm, amssymb, latexsym,graphicx}
\newcommand{\round}[1]{\ensuremath{\left\lfloor #1 \right\rceil}}
\newcommandx{\Sin}[2][2=2]{\sin^{#2}\term{#1}}
\usepackage[inline]{enumitem}
\usepackage{comment}
\usepackage{float}
\usepackage{xr-hyper} 
\usepackage{hyperref}

\usepackage{subcaption}

\usepackage[round]{natbib}

\begin{document}


\algnewcommand\algorithmicparfor{\textbf{parfor}}
\algnewcommand\algorithmicpardo{\textbf{do}}
\algnewcommand\algorithmicendparfor{\textbf{end\ parfor}}
\algrenewtext{ParFor}[1]{\algorithmicparfor\ #1\ \algorithmicpardo}
\algrenewtext{EndParFor}{\algorithmicendparfor}

\twocolumn[

\aistatstitle{Coresets for Data Discretization and Sine Wave Fitting}

\aistatsauthor{Alaa Maalouf \And Murad Tukan}% \And Daniel Kane \And Dan Feldman}

\aistatsaddress{ University of Haifa \And  University of Haifa} 

\aistatsauthor{Eric Price \And Daniel Kane \And Dan Feldman}

\aistatsaddress{University of Texas at Austin  \And University of California \And  University of Haifa} 
\runningauthor{Alaa Maalouf, Murad Tukan, Eric Price, Daniel Kane, Dan Feldman}
]

\begin{abstract}
%In the \emph{data discretization} problem, we are given a set $P$ of $n$ integers in $[N]:=\{1,\cdots,N\}$. 
%The goal is to compress the data by computing an integer $c\geq2$ that minimizes the rounding errors (distances) sum $cost(P,c):=\sum_{p\in P} |pc\ \ensuremath{\mathrm{mod}\  N}|$ over each number $p\in P$ to the nearest number in the ``low-resolution" grid $\{0,c,2c,3c,\cdots,\lfloor N/c\rfloor \cdot c\}$, under some constraints. For example, $\min_{c\in C}cost(P,c)+\lambda(c)$, where $C\subseteq [N]$ is a feasible set of solutions, and $\lambda$ is a given regularization function.
In the \emph{monitoring} problem, the input is an unbounded stream $P={p_1,p_2\cdots}$ of integers in $[N]:=\{1,\cdots,N\}$, that are obtained from a sensor (such as GPS or heart beats of a human). The goal (e.g., for anomaly detection) is to approximate the $n$ points received so far in $P$ by a single frequency $\sin$, e.g. $\min_{c\in C}cost(P,c)+\lambda(c)$, where $cost(P,c)=\sum_{i=1}^n \sin^2(\frac{2\pi}{N} p_ic)$, $C\subseteq [N]$ is a feasible set of solutions, and $\lambda$ is a given regularization function. 
For any approximation error $\varepsilon>0$, we prove that \emph{every} set $P$ of $n$ integers has a weighted subset $S\subseteq P$ (sometimes called core-set) of cardinality $|S|\in O(\log(N)^{O(1)})$ that approximates $cost(P,c)$ (for every $c\in [N]$) up to a multiplicative factor of $1\pm\varepsilon$. 
Using known coreset techniques, this implies streaming algorithms using only $O((\log(N)\log(n))^{O(1)})$ memory.  
Our results hold for a large family of functions. Experimental results and open source code are provided.
\end{abstract}

\section{INTRODUCTION AND MOTIVATION}
\textbf{Anomaly detection} is a step in data mining which aims to identify unexpected data points, events, and/or observations in data sets. 
For example, we are given an unbounded stream $P={p_1,p_2\cdots}$ of numbers that are obtained from a heart beats of a human (hospital patients) sensor, and the goal is to detect inconsistent spikes in heartbeats. This is crucial for proper examination of patients as well as valid evaluation of their health. Such data forms a wave which can be approximated using a \emph{sine} wave. 
Fitting a large data of this form (heart wave signals), will result in obtaining an approximation towards the distribution from which the data comes from. Such observation aids in detection of outliers or anomalies.

\begin{figure}[htb!]
\centering
\includegraphics[width=0.8\linewidth]{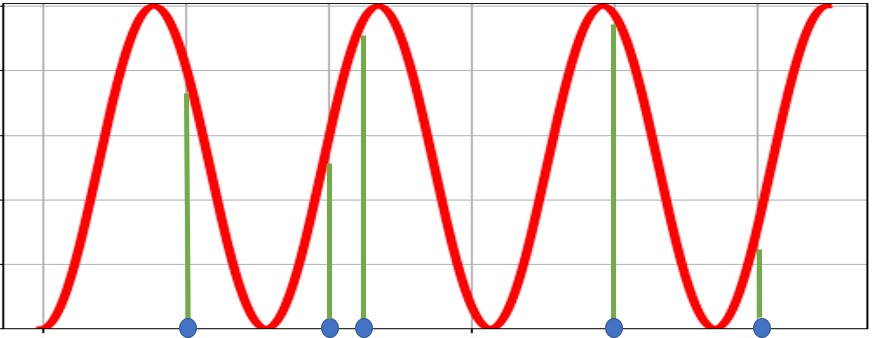}
\caption{\textbf{Sine fitting.} Given a set of integers $P$ (blue points on the $x$-axis), and $\sin^2(\cdot)$ wave (the red signal), then the cost of the \fft problem with respect to this input, is the sum of vertical distances between the points in $P$ (on the x-axis) and the sine signal (the sum of lengths of the green lines). The goal is to find the sine signal that minimizes this sum.}\label{fig:ourcost}
\end{figure}
Formally speaking, the anomaly detection problem can be stated as follows. Given a large positive number $N$, a set  $P \subseteq \br{1,2,\cdots,N}$ of $n$ integers, the objective is to fit a sine signal, such that the sum of the vertical distances between each $p \in P$ (on the $x$-axis) and its corresponding point $\Sin{\frac{2\pi}{N}pc}$ on the signal, is minimized; see Figure~\ref{fig:ourcost}. Hence, we aim to solve the following problem that we call the \emph{\fft problem}: %i.e., the sum $\Sin{\frac{2\pi}{N}pc}$ over every $p \in P$ will be minimized
\begin{equation}
\label{eq:our_cost}
\begin{split}
\min_{c\in C}\sum\limits_{p \in P} \sin^2\term{\frac{2\pi}{N} pc} + \lambda(c),
\end{split}
\end{equation}
where $C$ is the set of feasible solutions, and $\lambda$ is a regularization function to put constraints on the solution.
%This function measures the vertical distance between each point in $P$ where such point lie on the $x$-axis, to its nearest point on the \emph{sine} signal.
%This holds under the assumption that the maximal amplitude and shift of the sine signal are given. 

The generalized form of the fitting problem above was first addressed by~\cite{souders1994ieee} and later generalized by~\cite{ramos2008new}, where proper implementation have been suggested over the years~\cite{da2003new, chen2015improved, renczes2016efficient, renczes2021computationally}. In addition, the \fft problem and its variants gained attention in recent years in solving various problems, e.g., estimating the shift phase between two signal with very high accuracy~\cite{queiros2010cross}, characterizing data acquisition channels and analog to digital converters~\cite{pintelon1996improved}, high-accuracy sampling measurements of complex voltage ratio of sinusoidal signals~\cite{augustyn2018improved}, etc.

\textbf{Data discretization.} In many applications, we aim to find a proper choice of floating-point grid. For example, when we are given points encoded in $64$bits, and we wish to use a $32$ floating point grid. A naive way to do so is by simply removing the most/least significant $32$ bits from each point. However such approach results in losing most of the underlying structure that these points form, which in turn leads to unnecessary data loss.  
Instead, arithmetic modulo or sine functions that incorporate cyclic properties are used, e.g.,~\cite{naumov2018periodic, nagel2020up, gholami2021survey}. Such functions aim towards retaining as much information as possible when information loss is inevitable. 
This task serves well in the field of quantization~\cite{gholami2021survey}, which is an active sub-field in deep learning models.

 %The abstract and the intro deal with the same
%problem, but we gave different names for different usages in different fields. We will add these and more details.
%Our results can be extended to a form of the sine fitting where the exponent can be any z ≥ 1.

To solve this problem, we first find the sine wave that fits the input data using the cost function at~\eqref{eq:our_cost}. Then each point in the input data is projected to its nearest point from the set of roots of the signal that was obtained from the sine fitting operation; see Figure~\ref{fig:disc}.

All of the applications above, e.g., monitoring, anomaly detection, and data discretization, are problems that are reduced to an instance of the \fft problem. Although these problems are desirable, solving them on large-scale data is not an easy task, due to bounded computational power and memory. In addition, in the streaming (or distributed) setting where points are being received via a stream of data, fitting such functions requires new algorithms for handling such settings. To handle these challenges, we can use coresets. 

\subsection{Coresets} \label{sec:coresets}
Coreset was first suggested as a data summarization technique in the context of computational geometry~\citep{agarwal2004approximating}, and got increasing attention over recent years~\citep{broder2014scalable,nearconvex,huang2021novel,cohen2021improving,huang2020coresets,mirzasoleiman2020coresets}; for extensive surveys on coresets, we refer the reader to~\citep{feldman2020core, phillips2016coresets}, and~\cite{jubran2019introduction,maalouf2021introduction} for an introductory.

Informally speaking, a coreset is (usually) a small weighted subset of the original input set of points that approximates its loss for every feasible query $c$, up to a provable multiplicative error of $1 \pm \eps$, where $\eps \in (0, 1)$ is a given error parameter. 
Usually the goal is to have a coreset of size that is independent or near-logarithmic in the size of the input (number of points), in order to be able to store a data of the same structure (as the input) using small memory, and to obtain a faster time solutions (approximations) by running them on the coreset instead of the original data. Furthermore, the accuracy of existing (fast) heuristics can be improved by running them many times on the coreset in the time it takes for a single run on the original (big) dataset. 
Finally, since coresets are designed to approximate the cost of every feasible query, it can be used to solve constraint optimization problems, and to support streaming and distributed models; see details and more advantages of coresets in~\citep{feldman2020core}.

%Finally, coresets may be computed in time that is near-linear in the input, even for \emph{NP-hard} optimization problems. Existing heuristic or inefficient algorithms may then be applied many times on the small coreset to obtain improved or faster models in such cases.
%~\citep{buadoiu2008optimal,bachem2018one,broder2014scalable,nearconvex,tukan2020coresets,maalouf2019fast,maalouf2020tight,sets-clustering,jubran2021provably,tukan2021no}

In the recent years, coresets were applied to improve many algorithms from different fields e.g. logistic regression~\citep{huggins2016coresets,munteanu2018coresets,karnin2019discrepancy,nearconvex}, matrix approximation~\citep{feldman2013turning, maalouf2019fast,feldman2010coresets,sarlos2006improved,maalouf2021coresets}, decision trees~\citep{jubran2021coresets}, clustering~\citep{feldman2011scalable,gu2012coreset,lucic2015strong,bachem2018one,jubran2020sets, schmidt2019fair}, $\ell_z$-regression~\citep{cohen2015lp, dasgupta2009sampling, sohler2011subspace}, \emph{SVM}~\citep{har2007maximum,tsang2006generalized,tsang2005core,tsang2005very,tukan2021coresets}, deep learning models~\citep{baykal2018data,maalouf2021unified,liebenwein2019provable,mussay2021data}, etc.

\textbf{Sensitivity sampling framework.} A unified framework for computing coresets to wide range family of problems was suggested in~\citep{braverman2016new}. It is based on non-uniform sampling, specifically, sensitivity sampling. 
Intuitively, the sensitivity of a point $p$ from the input set $P$ is a number $s(p)\in [0,1]$ that corresponds to the importance of this point with respect to the other points, and the specific cost function that we wish to approximate; see formal details in Theorem~\ref{thm:coreset}. 
The main goal of defining a sensitivity is that with high probability, a non-uniform sampling from $P$ based on these sensitivities yields a coreset, where each point $p$ is sampled i.i.d. with a probability that is proportional to $s(p)$, and assigned a (multiplicative) weight which is inversely proportional to $s(p)$. The size of the coreset is then proportional to (i) the total sum of these sensitivities $t=\sum_{p\in P}s(p)$, and (ii) the VC dimension of the problem at hand, which is (intuitively) a complexity measure. 
In recent years, many classical and hard machine learning problems~\citep{braverman2016new,sohler2018strong,maalouf2020tight} have been proved to have a total sensitivity (and VC dimension) that is near-logarithmic in or even independent of the input size $\abs{P}$. %Hence, we can compute coresets for such problems using sensitivity sampling.

%In this paper, we prove that the total sensitivity of the \fft{} problem is near-logarithmic in the input size $\abs{P}$, and thus a small coreset for this problem exist.

\begin{figure}[htb!]
\centering
\includegraphics[width=0.8\linewidth]{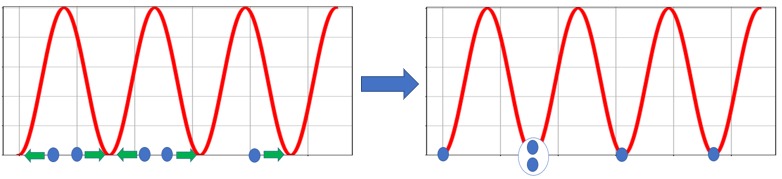}
\caption{\textbf{Discretization.} Given a set of points (blue points), we find a sine wave (red signal) that fits the input data. Then each input point is projected to its nearest point from the set of roots of the signal.}\label{fig:disc}
\end{figure}
\subsection{Our Contribution}
We summarize our contribution as follows.
\begin{enumerate}
\item[(i)] Theoretically, we prove that for every integer $N>1$, and every set $P\subseteq [N]$ of $n>1$ integers:
\begin{enumerate}
    \item The total sensitivity with respect to the \fft problem is bounded by $O(\log^4{N})$, and the VC dimension is bounded by $O(\log{(nN)})$; see Theorem~\ref{mainthm} and Claim~\ref{VcbOUND} respectively.
    
    \item For any approximation error $\eps>1$, there exists a coreset of size\footnote{$\Tilde{O}$ hide terms related to $\varepsilon$ (the approximation factor), and $\delta$ (probability of failure).} $\Tilde{O}\term{\log(N)^{O(1)}}$ (see Theorem~\ref{thm:mainthm} for full details) with respect to the \fft optimization problem.
    \end{enumerate}
    \item[(ii)] Experimental results on real world datasets and open source code~\citep{opencode} are provided.
\end{enumerate}
%To our knowledge, this is the first paper to show that ther exists a small coreset for 

\section{PRELIMINARIES}
In this section we first give our notations that will be used throughout the paper. We then define the sensitivity of a point in the context of the \fft problem (see Definition~\ref{def:sens}), and formally write how it can be used to construct a coreset (see Theorem~\ref{thm:coreset}). Finally we state the main goal of the paper.

\textbf{Notations.} Let $\INT$ denote the set of all positive integers, $[n] = \br{1,\ldots, n}$ for every $n \in \INT$, and for every $x \in \REAL$ denote the rounding of $x$ to its nearest integer by $\round{x}$ (e.g. $\round{3.2}=3$).

We now formally define the sensitivity of a point $p\in P$ in the context of the \fft problem.

\begin{definition}[\fft sensitivity]\label{def:sens}
Let $N>1$ be a positive integer, and let $P \subseteq [N]$ be a set of $n>1$ integers. For every $p \in P$, the \emph{sensitivity} of $p$ is defined as
$
\max_{c\in [N]}\frac{\sin^2(pc\cdot \frac{2\pi}{N})}{\sum_{q\in P}\sin^2(qc\cdot \frac{2\pi}{N})}.
$
\end{definition}

The following theorem formally describes how to construct an $\eps$-coreset via the sensitivity framework. We restate it from~\cite{braverman2016new} and modify it to be specific for our cost function.

\begin{theorem}\label{thm:coreset}
Let $N>1$ be a positive integer, and let $P \subseteq [N]$ be a set of $n>1$ integers. Let $s: P \to [0,1]$ be a function such that $s(p)$ is an upper bound on the sensitivity of $p$ (see Definition~\ref{def:sens}).
%($\max_{c\in [N]}\frac{\sin^2(pc\cdot 2\pi /N)}{\sum_{q\in P}\sin^2(qc\cdot 2\pi /N)}$) of $p$. 
Let $t = \sum_{p \in P} s(p)$ and $d'$ be the~\emph{VC dimension} of the \fft problem; see Definition~\ref{def:dimension}. 
Let $\eps, \delta \in (0,1)$, and let $S$ be a random sample of $\abs{S} \in O\term{\frac{t}{\varepsilon^2}\left(d'\log{t}+\log{\frac{1}{\delta}}\right)}$ i.i.d points from $P$, where every $p \in P$ is sampled with probability $s(p)/t$. Let $v(p) = \frac{t}{s(p)\abs{S}}$ for every $p \in S$. Then with probability at least $1-\delta$, we have that for every $c\in [N]$, we have
%\begin{equation*}
%\begin{split}
$\abs{1- \frac{\sum_{p\in S}v(p)\sin^2(pc\cdot \frac{2\pi} {N})}{\sum_{p\in P}\sin^2(pc\cdot \frac{2\pi}{N})}} \leq \eps .$
%\end{split}
%\end{equation*}
\end{theorem}

%Hence, in order to compute a coreset for a given problem, one needs to bound the total sensitivity and the VC dimension with respect to the given problem.

\textbf{Problem statement.} Theorem~\ref{thm:coreset} raises the following question: Can we bound the the total sensitivity and the VC dimension of the \fft problem in order to obtain small coresets?

Note that, the emphasis of this work is on the size of the coreset that is needed (required memory) to approximate the \fft cost function.

 %(or $S$ in the coreset case). 

%$\frac{c \log^4{N}}{\varepsilon^2}\left(\log(nN)\log{(\log{N})}+\log{\frac{1}{\delta}}\right)$
%$O\term{\frac{\log^4{N}}{\varepsilon^2}\term{\log{N}\log{\log{N}} + \log{(\log{N})}\log{(n)} + \log{\frac{1}{\delta}}}}$
%$\log{Nn} = \log{N} +\log{n}$
%To prove that a small coreset exists for a specif problem, we need to prove that  are bounded. 

%In what follows, we first bound that the total sensitivity,

%%%%%%%%%%%%%%%%%%%%%%%%%%%%%%%%%%%%%%%%%%%%%%%%%%%%%%%%%%%%%
\section{CORESET FOR SINE FITTING}
In this section we state and prove our main result. For brevity purposes, some proofs of the technical results have been omitted from this manuscript; we refer the reader to the supplementary material for these proofs. 
%In addition, for simplicity of notation, we assume that the weight of each point in the input set is $1$. However, it can be easily generalized to the case of weighted input.

Note that since the regularization function $\lambda$ at~\eqref{eq:our_cost} is independent of $P$, a $1\pm\eps$ multiplicative approximation of the $\sin^2(\cdot)$ terms at~\eqref{eq:our_cost}, yields a $1\pm\eps$ multiplicative approximation for the whole term in~\eqref{eq:our_cost}.

The following theorem summarizes our main result.
%Out main results is summarized in Theorem~\ref{thm:mainthm}.
\begin{theorem}[Main result: coreset for the \fft problem]\label{thm:mainthm}
Let $N>1$ be a positive integer, $P \subseteq [N]$ be a set of $n>1$ integers, and let $\eps,\delta \in (0,1)$.
Then, we can compute a pair $(S,v)$, where $S \subseteq P$, and $v:S\to [0,\infty)$, such that
\begin{enumerate}
    \item the size of $S$ is  polylogarithmic in $N$ and logarithmic in $n$, i.e.,  $$|S|\in  O\left(\frac{\log^4{N}}{\varepsilon^2}\left(\log(nN)\log{(\log{N})}+\log{\frac{1}{\delta}}\right)\right).$$
    \item with probability at least $1-\delta$, for every $c\in [N]$,
    $$\abs{1 -\frac{\sum_{p\in S}v(p)\sin^2(pc\cdot \frac{2\pi} {N})}{\sum_{p\in P}\sin^2(pc\cdot \frac{2\pi}{N})}} \leq \eps. $$
\end{enumerate}
\end{theorem}
To prove Theorem~\ref{thm:mainthm}, we need to bound the total sensitivity (as done in Section~\ref{sec:sensbound}) and the VC dimension (see Section~\ref{sec:vcbound}) of the \fft problem. 

\subsection{Bound On The Total Sensitivity}\label{sec:sensbound}
In this section we show that the total sensitivity of the \fft problem is small and bounded. Formally speaking,
\begin{theorem}\label{mainthm}
Let $N\geq 1$ and $P\subseteq [N]$. Then
\[
\sum_{p\in P}\max_{c\in [N]}\frac{\sin^2(pc\cdot \frac{2\pi}{N})}{\sum_{q\in P}\sin^2(qc\cdot \frac{2\pi}{N})}\in O(\log^4 N).
\]
\end{theorem}

We prove Theorem~\ref{mainthm} by combining multiple claims and lemmas. We first state the following as a tool to use the cyclic property of the sine function. %Before every claim/lemma, we give an intuitive explanation to it.
%We show that the sensitivity of each point can be computed by a query in a compact subset of $[N]$. However,

\begin{claim}
\label{clm:ax}
Let $a,b \in \INT$ be a pair of positive integers. Then for every $x \in \INT$, 
\[
\abs{\sin{\term{\frac{b\pi}{a} x}}} = \abs{\sin{\term{\frac{b\pi}{a} \term{\mod{x}[a]}}}}.
\]
\end{claim}
\begin{comment}

\begin{proof}
Put $x \in \INT$ and observe that
\begin{equation}
\label{eq:x_a}
x = \floor{\frac{x}{a}} a + \mod{x}[a].
\end{equation}

Thus, 
\begin{equation}
\label{eq:sin_prop_1}
\begin{split}
\abs{\sin{\term{\frac{b\pi}{a} x}}} &= \abs{\sin{\term{\frac{b\pi}{a} \floor{\frac{x}{a}}a + \frac{b\pi}{a}\mod{x}[a]}}}\\ 
&= \abs{\sin{\term{\floor{\frac{x}{a}} b\pi + \frac{b\pi}{a} \mod{x}[a]}}},
\end{split}
\end{equation}
where the first equality holds by~\eqref{eq:x_a}.

Using trigonometric identities, we obtain that
\begin{equation}
\label{eq:sin_prop_2}
\begin{split}
&\left|\sin{\term{\floor{\frac{x}{a}} b\pi + \frac{b\pi}{a} \mod{x}[a]}}\right| =\\ &\quad \left|\sin{\term{\floor{\frac{x}{a}} b\pi}} \cdot \cos{\term{\frac{b\pi}{a}\term{\mod{x}[a]}}} \right. \\
&\quad\left.+ \sin{\term{\frac{b\pi}{a}\term{\mod{x}[a]}}} \cdot \cos{\term{\floor{\frac{x}{a}} b\pi}}\right|.
\end{split}
\end{equation}

Since $\term{\floor{\frac{x}{a}} b\pi} \in \br{0, \pi, 2\pi, 3\pi, \cdots}$, we have that
\[
\sin{\term{\floor{\frac{x}{a}} b\pi}} = 0,
\]
and 
\[
\abs{\cos{\term{\floor{\frac{x}{a}} b\pi}}} = 1.
\]

By combining the previous equalities with~\eqref{eq:sin_prop_1} and~\eqref{eq:sin_prop_2}, Claim~\ref{clm:ax} follows.
\end{proof}
\end{comment}
We now proceed to prove that one doesn't need to go over all the possible integers in $[N]$ to compute a bound on the sensitivity of each $p \in P$, but rather a smaller compact subset of $[N]$ is sufficient.

\begin{lemma}\label{lem:sensitivity_query_bound}
Let $P \subseteq [N]$ be a set of $n$ integer points. For every $p \in P$, let 
$$C(p) = \br{c \in [N] \middle| \term{\mod{cp}} \in \left[\frac{N}{8}, \frac{3N}{8}\right]}.$$ Then for every $p \in P$,
\begin{equation}
\label{eq:sense_query_bound}
\max_{c\in [N]}\frac{\Sin{\frac{2\pi}{N}pc}}{\sum\limits_{q\in P}\Sin{\frac{2\pi}{N}qc}}\leq
4\cdot \max\limits_{c\in C(p)}\frac{\Sin{\frac{2\pi}{N}pc}}{\sum\limits_{q\in P}\Sin{\frac{2\pi}{N}qc}}.
\end{equation}
\end{lemma}
\begin{proof}
Put $p \in P$, and let $c^*\in[N]$ be an integer that maximizes the left hand side of~\eqref{eq:sense_query_bound} with respect to $p$, i.e.,
%\Sin{} -- squared
%\Sin{what u need}[] -- not squared
\begin{equation}
\label{maxx}
\max_{c\in [N]}\frac{\Sin{pc\cdot \frac{2\pi}{N}}}{\sum_{q\in P}\Sin{qc\cdot \frac{2\pi}{N}}}=\frac{\Sin{pc^*\cdot \frac{2\pi}{N}}}{\sum_{q\in P}\Sin{qc^*\cdot \frac{2\pi}{N}}}.
\end{equation}
If $c^*\in C(p)$ the claim trivially holds. Otherwise, we have $c^* \in [N]\setminus C(p)$,
and we prove the claim using case analysis: \begin{enumerate*}[label=Case~(\roman*)]
    \item $(\mod{c^*p})\in \left[0,\frac{N}{8}\right)$, \label{case:1} and
    \item $(\mod{c^*p})\in \left(\frac{N}{2}-\frac{N}{8},\frac{N}{2}\right]$\label{case:2}.
\end{enumerate*}%\textbf{(i)} $(\mod{c^*p})\in [0,N/b)$, and \textbf{(ii)} $(\mod{c^*p})\in (N/2-N/b,N/2)$.

\noindent\textbf{\ref{case:1}:} Let $b_2 \in [8]\setminus[3]$ be an integer, and let $z=\left\lceil \frac{N/b_2}{\mod{c^*p}}\right\rceil$.
We first observe that 
\begin{equation}
\label{eq:equality_z}
\begin{split}
z &= \left\lceil \frac{\frac{N}{b_2}}{\mod{c^*p}}\right\rceil\\
&=\frac{\frac{N}{b_2}}{\mod{c^*p}} + \frac{\mod{\term{-\frac{N}{b_2}}}[\term{\mod{c^*p}}]}{\mod{c^*p}},
\end{split}
\end{equation}
where the second equality hold by properties of the ceiling function. We further observe that,
\begin{equation}
\label{eq:bound_1_case_1}
\begin{split}
&z\term{\mod{c^*p}}\\
&= \frac{N}{b_2} + \mod{\term{-\frac{N}{b_2}}}[\term{\mod{c^*p}}]\\
&\in \left[\frac{N}{b_2}, \frac{N}{8} + \frac{N}{b_2}\right],
\end{split}
\end{equation}
where the first equality holds by expanding $z$ using~\eqref{eq:equality_z}, and the last inclusion holds by the assumption of~\ref{case:1}. Since $\left[\frac{N}{b_2}, \frac{N}{8} + \frac{N}{b_2}\right]$ is entirely included in $\left[\frac{N}{8}, \frac{3N}{8}\right]$, then it holds that $z\term{\mod{c^*p}} \in C(p)$. Similarly, one can show that $\mod{\term{zc^*p}} \in \left[\frac{N}{b_2}, \frac{N}{8} + \frac{N}{b_2}\right]$, which means that $zc^* \in C(p)$.

We now proceed to show that the sensitivity can be bounded using some point in $C(p)$. Since for every $x \in \left[0, \frac{\pi}{2}\right]$, $\sin{x} \leq x \leq 2\abs{\sin{x}}$, then it holds that 
\begin{equation}
\label{eq:bound_2_case_1}
\begin{split}
&\Sin{\frac{2\pi}{N}pc^*} = \Sin{\frac{2\pi}{N} \term{\mod{c^*p}}}\\
&\leq \term{\frac{2\pi}{N} \term{\mod{c^*p}}}^2\\
&= \term{\frac{2\pi}{N} z\term{\mod{c^*p}}}^2 \frac{1}{z^2}\\
&\leq \frac{4}{z^2} \Sin{\frac{2\pi}{N} z\term{\mod{c^*p}}} \\
&= \frac{4}{z^2} \Sin{\frac{2\pi}{N} zc^*p},
\end{split}
\end{equation}
where the first equality holds by plugging $a:=\frac{N}{2}$, $b := 1$ and $x := c^*p$ into Claim~\ref{clm:ax}, the first inequality holds since $\frac{2\pi}{N}\term{\mod{c^*p}} \in [0, \pi]$, the second equality holds by multiplying and dividing by $z$, the second inequality follows from combining the fact that $\frac{2\pi}{N} z\term{\mod{c^*p}} \leq \frac{\pi}{4}+\frac{\pi}{b_2} \leq \pi$ which is derived from~\eqref{eq:bound_1_case_1} and the observation that $2\abs{\sin{x}} \geq x$ for every $x \in \left[ 0, \frac{\pi}{2}\right]$, and the last equality holds by plugging $a:=\frac{N}{2}$, $b := z$ and $x := c^*p$ into Claim~\ref{clm:ax}.

In addition, it holds that for every $q \in P$
\begin{equation}%\abovedisplayskip=7pt plus2pt minus5pt
\label{eq:lower_bound_sin}
\begin{split}
&\Sin{\frac{2\pi}{N}\term{\mod{c^*q}}}\\
&\quad \geq \frac{1}{4} \term{\frac{2\pi}{N}\term{\mod{c^*q}}}^2\\
&\quad=\frac{\term{\frac{2z\pi}{N}\term{\mod{c^*q}}}^2}{4z^2}\\
&\quad\geq\frac{\Sin{\frac{2z\pi}{N} \term{\mod{c^*q}}}}{4z^2}
\\
&\quad=\frac{1}{4z^2}\Sin{\frac{2\pi}{N} zc^*q},
\end{split}
\end{equation}
where the first inequality holds by combining the assumption of~\ref{case:1} and the observation that $\abs{\sin{x}} \geq \frac{x}{2}$ for every $x \in \left[0, \frac{\pi}{2}\right]$, the first equality holds by multiplying and dividing by $z$, the second inequality holds by combining~\eqref{eq:bound_1_case_1} with the observation that $x \geq \abs{\sin{x}}$ for every $x \in [0, \pi)$ where in this context $x := \frac{2z\pi}{N}\term{\mod{c^*q}}$, and finally the last equality holds by plugging $a:= \frac{N}{2}$, $b:=z\pi$, and $x:= c^*q$ into Claim~\ref{clm:ax}.

Combining~\eqref{maxx},~\eqref{eq:bound_1_case_1},~\eqref{eq:bound_2_case_1} and~\eqref{eq:lower_bound_sin} yields that
\begin{equation*}
\begin{split}
&\max_{c \in [N]}\frac{\Sin{\frac{2\pi}{N}pc}}{\sum\limits_{q \in P}\Sin{\frac{2\pi}{N}qc}}\\ 
&\leq \max_{c \in [N]} \frac{\frac{16}{z^2} \Sin{\frac{2\pi z}{N}\term{\mod{pc}}}}{\frac{1}{z^2}\sum\limits_{q\in P} \Sin{\frac{2\pi z}{N}\term{\mod{qc}}}}\\
&= \max_{c \in [N]} \frac{16\Sin{\frac{2\pi }{N}zpc}}{\sum_{q\in P}\Sin{\frac{2\pi }{N}zqc}}
= \max_{\hat{c} \in C(p)} \frac{16\Sin{\frac{2\pi }{N}\hat{c}p}}{\sum_{q\in P}\Sin{\frac{2\pi }{N}\hat{c}q}},
\end{split}
\end{equation*}
where last equality holds from combining~\eqref{maxx} and $zc^* \in C(p)$.

\noindent\textbf{\ref{case:2}:} Let $c^\prime=N-c^*$, and note that $c^\prime\in[N]$. For every $q\in P$,
\[
\abs{\Sin{c^\prime q\cdot 2\pi/N}[]} =\abs{\Sin{c^*q\cdot 2\pi/N}[]}.
\]
We observe that
\begin{align*}
&(\mod{c^\prime p}) = \mod{(N - c^*)p}\\
&\quad\quad= \mod{\left(\mod{Np} + \mod{(-c^*p)}\right)}\\
&\quad\quad=\mod{\left( 0 + \mod{(-c^*p)} \right)} \\
&\quad\quad= \mod{(-c^*p)}\\
&\quad\quad \leq \mod{-\left( N/2 - N/8 \right)}\\
&\quad\quad= \mod{N/8} - \mod{N} = N/8.
\end{align*}
Hence, the proof of Claim~\ref{lem:sensitivity_query_bound} in~\ref{case:2} follows by replacing $c^*$ with $c'$ in~\ref{case:1}.
\end{proof}

%In this section we prove that for any set $P\subseteq[N]$, the total sensitivity $\sum_{p\in P}s(p)$ is bounded by $O(\log(N))$ (see Theorem~\ref{mainthm}), and thus (by Theorem~\ref{}), there exists a coreset of size logarithmic in $N$.

% \begin{proof}
% We denote $D(p,c):=\sin^2(pc\cdot 2\pi /N)$ for every $p,c\in[N]$.
% Let $p\in P$, $b=8$, and
% \begin{equation}\label{cP}
% C(p):=\br{c\in [N]\mid (\mod{cp}) \in \left[\frac{N}{b}, \frac{N}{2}-\frac{N}{b}\right]}.
% \end{equation}
% The proof of Theorem~\ref{mainthm} relies on multiple claims, we start by proving the following claim.

In what follows, we show that the sensitivity of each point $p \in P$ is bounded from above by a factor that is proportionally  polylogarithmic in $N$ and inversely linear in the number of points $q \in P$ that are not that far from $p$ in terms of arithmetic modulo.

\begin{lemma}
\label{lem:bounding_c(p)}
Let $C(p)$ be as in Lemma~\ref{lem:sensitivity_query_bound} for every $p \in P$, and let
\begin{equation}
\label{eq:gpP}
\begin{split}
g(p,P)&=\min_{c\in C(p)}\left|\left\lbrace q\in P:  (\mod{cq})\in \right.\right.\\
&\left.\left.\left[\frac{N}{16\log N}, \frac{N}{2}-\frac{N}{16\log N}\right] \right\rbrace\right|.
\end{split}
\end{equation}
Then for every $p \in P$,
\[
\max_{c\in C(p)}\frac{\Sin{pc\cdot \frac{2\pi}{N}}}{\sum_{q\in P}\Sin{qc\cdot \frac{2\pi}{N}}} \in O(\log^2 N)\cdot \frac{1}{g(p, P)}.
\]
\end{lemma}
\begin{proof}
Put $p \in P$, $c\in C(p)$, and let $P^\prime = \br{q \in P : \Sin{qc\cdot \frac{2\pi}{N}} \geq 1/(8\log{N})^2}$. 

First we observe that for every $q \in P$ such that $\Sin{qc\cdot \frac{2\pi}{N}} \geq 1/(8\log{N})^2$, it is implied that $\mod{(qc)} \geq \frac{N}{16\log{N}}$. By the cyclic property of $\sin$, it holds that $\mod{(qc)} \leq \frac{N}{2} - \frac{N}{16\log{N}}$.

Combining the above with the fact that $\Sin{pc\cdot \frac{2\pi}{N}}\leq 1$, yields that
\[
\begin{split}
&\frac{\Sin{pc\cdot \frac{2\pi}{N}}}{\sum_{q\in P}\Sin{qc\cdot \frac{2\pi}{N}}} \\
& \quad \leq\frac{1}{\sum_{q\in P}\Sin{qc\cdot \frac{2\pi}{N}}} \leq \frac{1}{\sum_{q\in P^\prime}\Sin{qc\cdot \frac{2\pi}{N}}}\\
&\quad \leq \frac{1}{\abs{P^\prime}/\term{64\log^2N}}
\leq \frac{64\log^2N}{g(p,P)},
\end{split}
\]
where the second inequality follows from $P^\prime \subseteq P$, and the last derivation holds since $g(p,P) \leq \abs{P^\prime}$ which follows from~\eqref{eq:gpP}.
\end{proof}

The bound on the sensitivity of each point $p \in P$ (from the Lemma~\ref{lem:bounding_c(p)}) still requires us to go over all possible queries in $C(p)$ to obtain the closest points in $P$ to $p$. Instead of trying to bound the sensitivity of each point by a term that doesn't require evaluation over every query in $C(p)$, we will bound the total sensitivity in a term that is independent of $C(p)$ for every $p \in P$. This is done by reducing the problem to an instance of the expected size of independent set of vertices in a graph (see Claim~\ref{clm:bound_gPp}). First, we will use the following claim to obtain an independent set of size polylogarithmic in the number of vertices in any given directed graph. 

\begin{claim}\label{three}
Let $G$ be a directed graph with $n$ vertices. Let $d_i$ denote the out degree of the $i$th vertex, for $i=1,\cdots, n$. Then there is an independent set of vertices $Q$ in $G$ such that
$|Q| \in \frac{\Theta(1)}{\log N}\cdot \sum_{i=1}^n \frac{1}{d_i+1}.$
\end{claim}
\begin{proof}
Let $V$ denote the set of vertices of $G$, and let $E$ denote the set of edges of $G$. Partition the vertices of $G$ into $O(\log n)$ induced sub-graphs, where each vertex in the $j$th sub-graph $H_j$ has out degree in $[2^{j-1},2^{j}-1]$ for any non-negative integer $j \leq \round{\log{\term{\max\limits_{i \in [n]} d_i}}}$. %$j=1,2,\ldots$.
Let $H$ denote the sub-graph with the largest number of vertices.
Pick a random sample $S$ of $|V|^2/(2|E|)$ nodes from $V$.
The expected number of edges in the induced sub-graph of $H$ by $S$ is bounded by $|V|^2/(4|E|)$.
Let 
\begin{align*}
T= &\{v\mid (v,u) \text{ is an edge of the sub-graph of }H\\& \text{ induced by }S\}.  
\end{align*}
By Markov inequality, with probability at least $1/2$ we have $|T|\leq V^2/2$.
Assume that this event indeed holds.  Hence, the sub-graph of $H$ that is induced by $S\setminus T$ is an independent set of $H$ with $|S\setminus T|= |V|^2/2|E|-V^2/4|E|\geq V^2/4|E|$ nodes.
Since $|E| \in O\left(2^j \abs{V} \right)$ we have $|S_j\setminus T_j|=|V|/2^j\in\sum_{v\in V_j}\Theta(1)/d_v$.
Hence, $\sum_j \abs{S_j\setminus T_j} \in \sum_{v\in V}\Theta(1)/d_v$.
By the pigeonhole principle, there is $j$ such that $\abs{S_j\setminus T_j}$ can be bounded from below by $\sum_{v\in V}\Theta(1)/d_v/ O(\log n)$.
\end{proof}

\begin{claim}\label{clm:bound_gPp}
There is a set $Q\subseteq P$ such that $g(q,Q)=1$ for every $q\in Q$, and
$\sum_{p\in P}\frac{1}{g(p, P)}\in |Q|\cdot \Theta\term{\log N}.$
\end{claim}
\begin{proof}
Let $b_2$ be defined as in the proof of Lemma~\ref{lem:sensitivity_query_bound}. For every $p\in P$, let $g^{-1}(p,P)\in C(p)$ such that
\begin{align*}
&g^{-1}(p,P)\in \\
&\quad \arg\min_{c\in C(p)} \left| \left\lbrace q\in P\setminus{\br{p}}: \Sin{qc\cdot \frac{2\pi}{N}}\geq \frac{b2}{N}\right\rbrace\right|.
\end{align*}
Let $G=(P,E)$ denote the directed graph whose vertices are the integers in $P$, and whose edges are
\begin{equation}\label{Edef}
\begin{split}
E &:=\bigg\{ (p,q)\in P\times P : p\neq q \text{ and }\\
&\Sin{\frac{2\pi}{N} qg^{-1}(p,P)}\geq \frac{1}{b_3} \bigg\}.
\end{split}
\end{equation}
The out degree of a vertex $p$ of $G$ is $g(p,P)-1$. %Let $Q$ be an independent set of $G$.
By Claim~\ref{three}, there is independent set $Q$ of $G$ such that
$
 \sum_{i=1}^n \frac{1}{g(p,P)} \in |Q|\cdot \Theta(\log N). 
$
Since $Q$ is independent set, for every $q\in Q$ we have $g(q,Q)=1$.
\end{proof}

The following claim serves to bound the size of the independent set, i.e., for every point $p \in P$ in the set, $g(p,P) = 1$.
\begin{claim} \label{clm:bound_Q}
Let $Q\subseteq[N]$ such that $g(q,Q)=1$ for every $q\in Q$. Then
$ |Q|\leq \log N. $
\end{claim}

Finally, combining Lemma~\ref{lem:sensitivity_query_bound}, Lemma~\ref{lem:bounding_c(p)}, Claim~\ref{clm:bound_gPp} and Claim~\ref{clm:bound_Q}, satisfies Theorem~\ref{mainthm} which presents a bound on the total sensitivity with respect to the $\Sin{\cdot}$ cost function.

\subsection{Bound on The VC Dimension}\label{sec:vcbound}
First, we define VC dimension with respect to the \fft problem.   
\begin{definition}[VC-dimension~\citep{braverman2016new}]
\label{def:dimension}
Let $N>1$ be a positive integer, $P\subset [N]$ be a set of $n$ integer, and let $r \in [0,\infty)$, we define 
\[
\RANGES(x,r) = \br{p \in P \mid f(p,x) \leq r},
\]
for every $x \in [N]$ and $r \geq 0$. The dimension of the \fft problem is the size $\abs{S}$ of the largest subset $S \subset P$ such that
\[
\abs{\br{S \cap \RANGES(x,r) \mid x \in [N], r \geq 0 }} = 2^{\abs{S}}.
\]
%where $\abs{ A }$ denotes the number of points in $A$ for every $A \subseteq \REAL^d$.
\end{definition}

\begin{lemma}[Bound on the VC dimension of the \fft problem]\label{VcbOUND}
Let $n,N \geq 1$ be a pair of positive integers such that $n \leq N$, and let $P \subseteq [N]$ be a set of $n$ points. 
Then the VC dimension of the \fft problem with respect to $P$ and $N$ is $O\term{\log\term{nN}}$. 
\end{lemma}

\begin{proof}
We note that the VC-dimension of the set of classifiers that output the sign of a sine wave parametrized by a single parameter (the angular frequency of the sine wave) is infinite. 
%Fist, we note that the VC dimension of squared sine function (or even sine function) is proven to be $\infty$~\citep{}\todo{Cite proper paper}. 
However since our query space is bounded, i.e., every query is an integer in the range $[1,N]$, then the VC dimension is bounded as follows. First, let $D(p,x)= \sin^2( px \cdot \frac{2\pi}{N})$ for every $p\in P$, and $x\in [n]$. We observe that for every $p \in P$ and $x \in [N]$, $D(p,x) \leq 1$. Hence, for every $x \in [N]$ and $r \in [0,\infty)$ it holds that 
$\br{\RANGES(x,r) \middle| r \geq 0} =  \br{\RANGES(x,r) \middle| r \in [0,1]},$
where $\RANGES(x,r) = \br{p \in P \mid d(p,x) \leq r}$ is defined as in Definition~\ref{def:dimension}. Secondly, by the definition of $\RANGES$, we have that for every pair of $r_1,r_2 \in [0,1]$ and $x \in [N]$ where $r_2 \geq r_1$, $\RANGES\term{x,r_2} = \bigcup\limits_{r \in \left[r_1, r_2\right]} \RANGES\term{x,r}.$

This yields that $\abs{\br{\RANGES\term{x,r} \middle| r \in [0,1]}} \leq n$ for any $x \in [N]$, which consequently means that
$
\abs{\br{\RANGES\term{x,r} \middle| x \in [N], r \geq 0}} = nN,
$
since $x \in [N]$ is an integer, and each such $x$ would create a different set of $n$ subsets of $P$. Thus we get that $\forall \, S \subseteq P$:
$
\abs{\br{S \cap \RANGES\term{x,r} \middle| x \in [N], r \geq 0}} \leq nN = 2^{\log\term{nN}}.
$

The claim then follows since the above inequality states that the VC dimension is bounded from above by $\log\term{nN}$.
\end{proof}

\section{REMARKS AND EXTENSIONS}
In this section briefly discuss several remarks and extensions of our work. % where in the appendix, we give 

\textbf{Parallel implementation. } Computing the sensitivities for $n$ input points requires $O(Nn)$ time, this is by computing $cost(c):= {\sum_{q\in P}\sin^2( qc \cdot \frac{2\pi}{N})}$ for every $c\in [n]$, and then bounding the sensitivity for very $p\in P$ by iterating over all queries $c\in [N]$, and taking the one which maximizes its term. 
However, this can be practically improved by applying a distributed fashion algorithm. 
Notably, one can compute the cost ${\sum_{q\in P}\Sin{qc\cdot \frac{2\pi}{N}}}$ of every query $c\in [N]$ independently from all other queries in $[N]$, similarly, once we computed the cost of every query $c\in[N]$, the sensitivity of each point $p\in P$ can be computed independently from all of the other points. Algorithm~\ref{alg:sens} utilises these observations: It receives as input an integer $N$ which indicates the query set range, a set $P\subset [N]$, and an integer $M$ indicating the number of machines given to apply the computations on. Algorithm~\ref{alg:sens} outputs a function $s:P\to (0,\infty)$, where $s(p)$ is the sensitivity of $p$ for every $p\in P$. 

\setcounter{AlgoLine}{0}
\begin{algorithm}
\caption{$\textsc{Calculate-Sensitivities}(P,N,M)$\label{alg:sens}}
\SetKwInOut{Input}{Input}
\SetKwInOut{Output}{Output}
\Input{An integer $N>1$, a set $P\subset [N]$ of $n>1$ integers, and an integer $M\geq 1$.}
\Output{A function $s:P\to (0,\infty)$, where for every $p\in P:$ $s(p)$ is the sensitivity of $p$.}

$C_1,\cdots,C_M:=$ a partition of $[N]$ into $M$ disjoint subsets, each contains at most $\ceil{N/M}$ integers from $[N]$. \Comment{In some cases, the last set $C_M$ might be empty.}

$P_1,\cdots,P_M:=$ a partition of $P$ into $M$ disjoint subsets, each contains at most $\ceil{n/M}$ integers from $P$. \Comment{In some cases, the last set $P_M$ might be empty.}

\For{every $i\in [M]$, in distributed manner}
{

\For{every $c\in C_i$}{

Set $cost(c):= {\sum_{q\in P}\Sin{qc\cdot \frac{2\pi}{N}}}$
}
}

\For{every $i\in [M]$, in distributed manner}
{

\For{every $p\in P_i$}{

Set $s(p):= \max_{c\in [N]} \frac{\Sin{pc\cdot \frac{2\pi}{N}}}{cost(c)}$
}
}

\Return $s$

\end{algorithm}

\textbf{Extension to high dimensional data.} Our results can be easily extended to the case where (i) the points (of $P$) lie on a polynomial grid of resolution $\Delta>0$ of any dimension $d\geq 1$, and (ii) they are further assumed to be contained inside a ball of radius $N>0$. Note that, such assumptions are common in the coreset literature, e.g., coresets for protective clustering~\cite{edwards2005no}, relu function~\cite{mussay2021data}, and logistic regression~\cite{tolochinsky2018generic}. 
The analysis with respect to the sensitivity can be directly extended, and the VC dimension is now bounded by $O\left(d \log\left(\frac{N}{d} \Delta n\right)\right)$. Both claims are detailed at Section~\ref{sec:highd} of the appendix.

\textbf{Approximating the optimal solution via coresets. } Let $N>1$ be an integer, $P\subset [N]$, and let $(S,v)$ be a coreset for $P$ as in Theorem~\ref{thm:coreset}. Let $p^* \in \arg\min_{c\in [N] }\sum_{p \in P}\sin^2( pc\cdot \frac{2\pi} { N})$ and $c^* \in \arg\min_{c\in [N]}\sum_{p\in S}v(p)\sin^2(pc\cdot \frac{2\pi} {N})$ be the optimal solutions on the input and its coreset, respectively, then 
$ \sum_{p \in P}\sin^2(  pc^* \cdot {2\pi}{N})\leq (1+\varepsilon) \sum_{p \in P}\sin^2( pp^* \cdot {2\pi}{ N})$.

%$\abs{1- \frac{\sum_{p\in S}v(p)\sin^2(pc\cdot \frac{2\pi} {N})}{\sum_{p\in P}\sin^2(pc\cdot \frac{2\pi}{N})}} \leq \eps .$

\section{EXPERIMENTAL RESULTS}
\begin{figure*}[htb!]
\centering
\includegraphics[width=.32\textwidth]{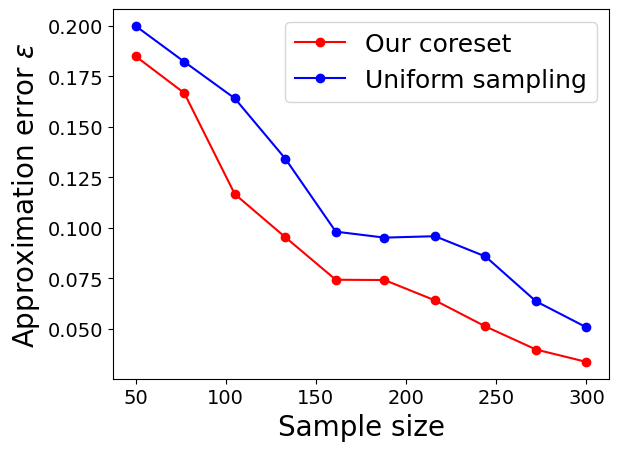}
\includegraphics[width=.32\textwidth]{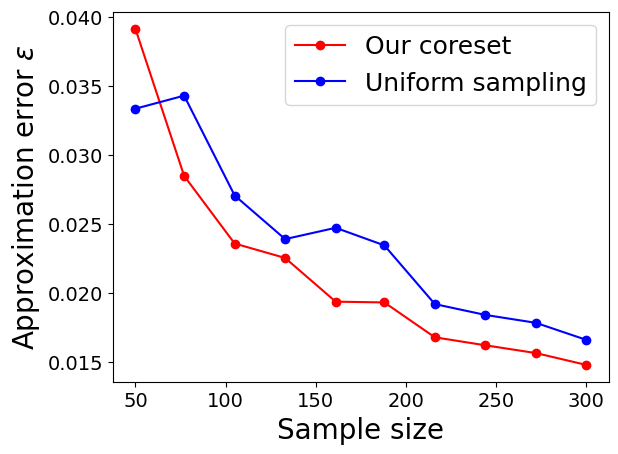}
\includegraphics[width=.32\textwidth]{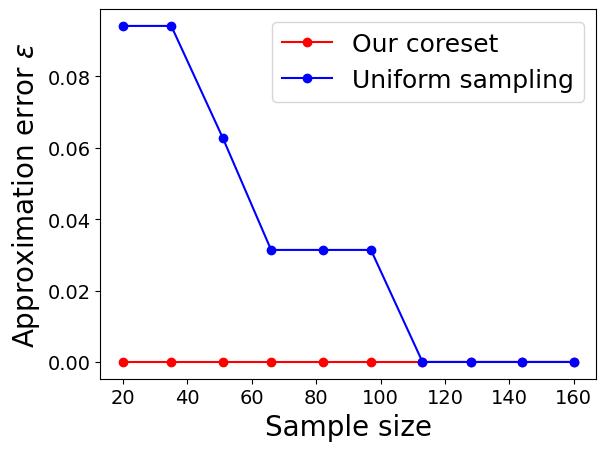}
\caption{Optimal solution approximation error: The x axis is the size of the chosen subset, the y axis is the optimal solution approximation error. Datasets, from left to right, (i)-(1),  (i)-(2), and (iii).}
\label{fig:optmal soltuion}
\end{figure*}
\begin{figure*}[htb!]
\includegraphics[width=.32\textwidth]{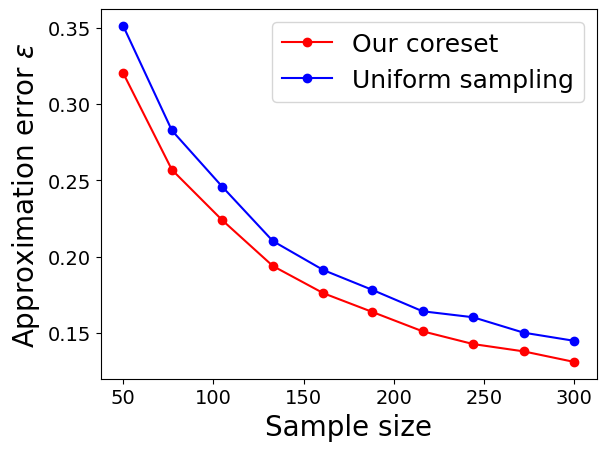}
\includegraphics[width=.32\textwidth]{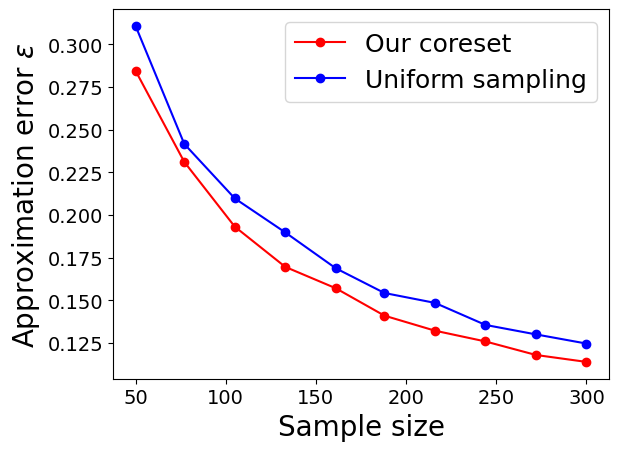}
\includegraphics[width=.32\textwidth]{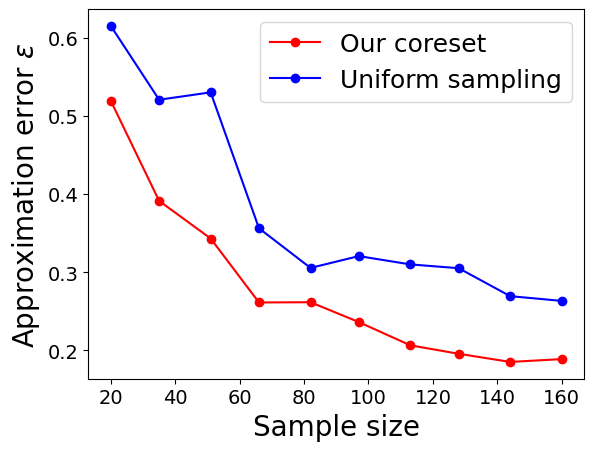}
\caption{Maximum approximation error: The x axis is the size of the chosen subset, the y axis is the maximum approximation error across the whole set of queries.  Datasets, from left to right, (i)-(1),  (i)-(2), and (iii). }
\label{fig:max error}
%\end{subfigure}
\end{figure*}

In what follows we evaluate our coreset against uniform sampling on real-world datasets. 

\textbf{Software/Hardware.} Our algorithms were implemented in Python 3.6~\citep{10.5555/1593511} using \say{Numpy}~\citep{oliphant2006guide}. Tests were performed on $2.59$GHz i$7$-$6500$U ($2$ cores total) machine with $16$GB RAM.

\subsection{Datasets And Applications}\label{sed:datasets}
%We used the following datasets. % were used for our experiments: %mostly from UCI machine learning repository~\citep{Dua:2019}:
\begin{enumerate}[label=(\roman*)]
    \item Air Quality Data Set~\citep{de2008field}, which contains $9,358$ instances of hourly averaged responses from an array of $5$ metal oxide chemical sensors embedded in an Air Quality Chemical Multisensor Device. We used two attributes (each as a separate dataset) of hourly averaged measurements of (1) tungsten oxide - labeled by (i)-(1) in the figures, and (2) NO2 concentration - labeled by (i)-(2). 
    Fitting the sine function on each of these attributes aids in understanding their underlying structure over time. This helps us in finding anomalies that are far enough from the fitted sine function. Finding anomalies in this context could indicate a leakage of toxic gases. Hence, our aim is to monitor their behavior over time, while using low memory to store the data. 
    \item Single Neuron Recordings~\citep{singlenueron} acquired from a cat's auditory-nerve fiber. The dataset has $127,505$ samples and the goal of \fft with respect to such data is to infer cyclic properties from neuron signals which will aim in further understanding of the wave of a single neuron and it's structure.
    
    \item Dog Heart recordings of heart ECG~\citep{dogHeart}. The dataset has $360,448$ samples. We have used the absolute values of each of the points corresponding to the ``electrocardiogram'' feature which refers to the ECG wave of the dog's heart. The goal of \fft on such data is to obtain the distribution of the heart beat rates. This aids to detects spikes, which could indicate health problems relating to the dog's heart.
\end{enumerate}
 %Regarding Figure~\ref{} on datase

\subsection{Reported Results} 

%We apply various experiments with respect to the \fft problem, each on multiple datasets. First, we report and compare the following
\textbf{Approximation error.} We iterate over different sample sizes, where at each sample size, we generate two coresets, the first is using uniform sampling and the latter is using sensitivity sampling. For every such coreset $(S,v)$, we compute and report the following.
\begin{enumerate}[label=(\roman*)]
   \item The optimal solution approximation error, i.e., we find $c^* \in \arg\min_{c \in C} \sum_{p \in S} v(p)\sin^2(\frac{2\pi}{N} \cdot pc)$. Then the approximation error $\varepsilon$ is set to be $\frac{\sum_{p \in P} \sin^2(\frac{2\pi}{N}  pc^*)}{\min_{c \in C}  \sum_{p \in P} \sin^2(\frac{2\pi}{N}  pc)}-1$; see Figure~\ref{fig:optmal soltuion}.
   %\item the average approximation error of the coreset over all query in the query set against uniform sampling; See Figures~\ref{}--\ref{}.
   \item The maximum approximation error of the coreset over all query in the query set, i.e., $\max_{c\in C} \abs{1 - \frac{\sum_{p \in S} v(p)\sin^2(\frac{2\pi}{N}  pc)}{\sum_{p \in P} \sin^2(\frac{2\pi}{N}  pc)}}$; see Figure~\ref{fig:max error}.
\end{enumerate}
The results were averaged across $32$ trials.  As can be seen in Figures~\ref{fig:optmal soltuion} and~\ref{fig:max error}, the coreset in such context (for the described applications in Section~\ref{sed:datasets}) encapsulates the structure of the dataset and approximate the datasets behavior. Our coreset obtained consistent smaller approximation errors in almost all the experiments in both experiments than those obtained by uniform sampling. Observe that our advantage on Dataset~(iii) is much more significant than the others as this dataset admits a clear periodic underlying structure. 
Note that, in some cases the coreset is able to encapsulate the entirety of the underlying structure at small sample sizes much better than uniform sampling due to its sensitivity sampling. This means that the optimal solution approximation error in practice can be zero; see the rightmost plot in Figure~\ref{fig:optmal soltuion}.

\begin{figure*}[htb!]
\includegraphics[width=\textwidth]{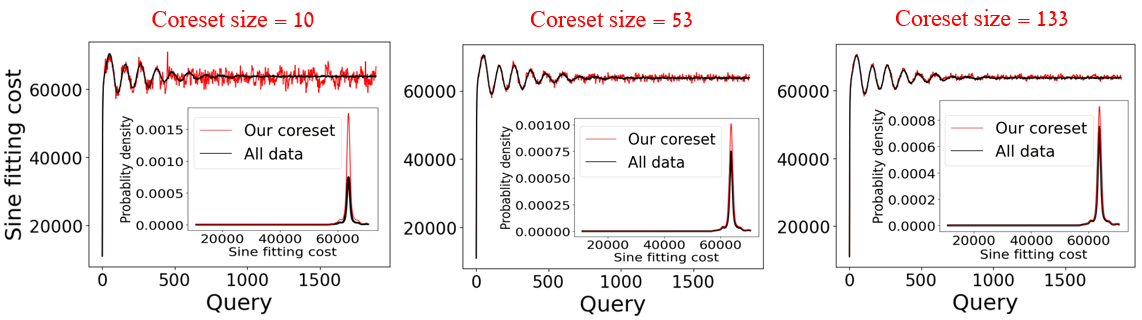}
\caption{\fft cost as a function of the given query. Dataset (ii) was used.}
\label{fig:sine-fit-spike}
%\end{subfigure}
\end{figure*}

\begin{figure*}[htb!]
\includegraphics[width=\textwidth]{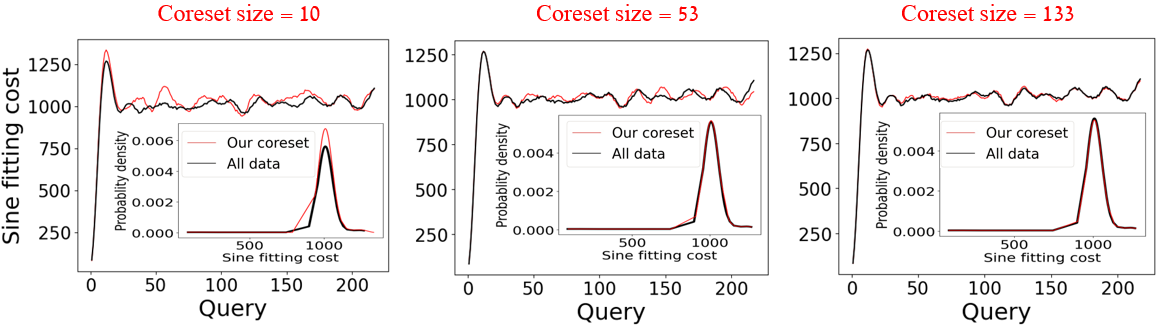}
\caption{\fft cost as a function of the given query. Dataset (iii) was used.}
\label{fig:sine-fit-dog}
%\end{subfigure}
\end{figure*}

\textbf{Approximating the Sine function's shape and the probability density function of the costs.} In this experiment, we visualize the \fft cost as in~\eqref{eq:our_cost} on the entire dataset over every query in $[N]$ as well as visualizing it on our coreset. As depicted in Figures~\ref{fig:sine-fit-spike} and~\ref{fig:sine-fit-dog}, the large the coreset size, the smaller the deviation of both functions. This proves that in the context of \fft, the coreset succeeds in retaining the structure of the data up to a provable approximation. In addition, due to the nature of our coreset construction scheme, we expect that the distribution will be approximated as well. This also can be seen in Figure~\ref{fig:sine-fit-spike} and~\ref{fig:sine-fit-dog}. 
Specifically speaking, when the coreset size is small, then the deviation (i.e., approximation error) between the cost of~\eqref{eq:our_cost} on the coreset from the cost of~\eqref{eq:our_cost} on the whole data, will be large (theoretically and practically), with respect to any query in $[N]$. As the coreset size increases, the approximation error decreases as expected also in theory. This phenomenon is observed throughout our experiments, and specifically visualized at Figures~\ref{fig:sine-fit-spike} and~\ref{fig:sine-fit-dog} where one can see that the alignment between probability density functions with respect to the coreset and the whole data increases with the coreset size. Note that, we used only $2000$ points from Dataset (iii) to generate the results presented at Figure~\ref{fig:sine-fit-dog}.

% \textbf{.} In this experiment, we show that for our coreset $(S,v)$, the histogram of $\br{\sum\limits_{p \in S} v(p)\Sin{2\pi px/N} \middle| x \in [N]}$ is close to that of the histogram of $\br{\sum\limits_{p \in P} \Sin{2\pi px/N} \middle| x \in [N]}$. As shown in Figure~\ref{fig:hists}, the large the sample size (coreset size), the more aligned the histograms are.

\section{CONCLUSION, NOVELTY, AND FUTURE WORK}

\textbf{Conclusion.} In this paper, we proved that for every integer $N>1$, and a set $P\subset [N]$ of $n>1$ integers, we can compute a coreset of size $O(\log(N)^{O(1)})$ for the \fft problem as in~\eqref{eq:our_cost}. 
Such a coreset approximates the \fft cost for every query $c$ up to a $1\pm\eps$ multiplicative factor, allowing us to support streaming and distributed models.
Furthermore, this result allows us to gain all the benefits of coresets (as explained in Section~\ref{sec:coresets}) while simultaneously maintaining the underlying structure that these input points form as we showed in our experimental results. 

%Then, we suggested%Then we suggested an extension to our coreset, w
%Then, we showed how to compute a coreset of the same size ($O(\log(N)^{O(z)})$) for the \fft problem to the power of $z\geq 1$.
%this is by a reduction to the \fft problem. 

%us to (a) store a much smaller number of integer in memory, (b) approximate the \fft cost for any integer $c\in [N]$ in a fast time, (c) support streaming and distributed data models, and (d) improve the accuracy of existing (fast) heuristics by running them many times on the coreset instead of one time on the original data, all this while maintaining the underlying structure that these input points form as we showed in our experimental results. 

\textbf{Novelty.} The proofs are novel in the sense that the used techniques vary from different fields that where not previously leveraged in the context of coresets, e.g., graph theory, and trigonometry. 
Furthermore to our knowledge, our paper is the first to use sensitivity to obtain a coreset for problems where the involved cost function is trigonometric, and generally functions with cyclic properties. We hope that it will help open the door for more coresets in this field.

\textbf{Future work} includes (i) suggesting a coreset for a high dimensional input, (ii) computing and proving a lower bound on the time it takes to compute the coreset, (iii) extending our coreset construction to a generalized form of cost function as in~\citep{souders1994ieee,ramos2008new}, and (iv) discussing the applicability of such coresets in a larger context such as quantization~\citep{hong2022daq,zhou2018adaptive,park2017weighted} of deep neural networks while merging it with other compressing techniques such as pruning~\citep{liebenwein2019provable,baykal2018data} and low-rank decomposition~\citep{tukan2021no,maalouf2020deep,liebenwein2021compressing}, and/or using it as a prepossessing step for other coreset construction algorithms that requires discretization constraints on the input, e.g.,~\citep{varadarajan2012near}. 
\label{sec:conclutions}

\clearpage

\section{ACKNOWLEDGEMENTS}
This work was partially supported by the Israel National Cyber Directorate via the BIU Center for Applied Research in Cyber Security.

\bibliographystyle{apalike}
\bibliography{main.bib}

%%%%%%%%%%%%%%%%%%%%%%%%%%%%%%%%%%%
%%%%%% SUPPLEMENT (OPTIONAL) %%%%%%
%%%%%%%%%%%%%%%%%%%%%%%%%%%%%%%%%%%

\clearpage
\appendix

\thispagestyle{empty}

% For one-column format, uncomment the following:
\onecolumn \makesupplementtitle
% For two-column format, uncomment the following:
%\twocolumn[ \makesupplementtitle ]
\section{PROOF OF TECHNICAL RESULTS}
\subsection{Proof of Claim~\ref{clm:ax}}
\begin{proof}
Put $x \in \INT$ and observe that
\begin{equation}
\label{eq:x_a}
x = \floor{\frac{x}{a}} a + \mod{x}[a].
\end{equation}

Thus, 
\begin{equation}
\label{eq:sin_prop_1}
\begin{split}
\abs{\sin{\term{\frac{b\pi}{a} x}}} &= \abs{\sin{\term{\frac{b\pi}{a} \floor{\frac{x}{a}}a + \frac{b\pi}{a}\mod{x}[a]}}}\\ 
&= \abs{\sin{\term{\floor{\frac{x}{a}} b\pi + \frac{b\pi}{a} \mod{x}[a]}}},
\end{split}
\end{equation}
where the first equality holds by~\eqref{eq:x_a}.

Using trigonometric identities, we obtain that
\begin{equation}
\label{eq:sin_prop_2}
\begin{split}
&\left|\sin{\term{\floor{\frac{x}{a}} b\pi + \frac{b\pi}{a} \mod{x}[a]}}\right| =\\ &\quad \left|\sin{\term{\floor{\frac{x}{a}} b\pi}} \cdot \cos{\term{\frac{b\pi}{a}\term{\mod{x}[a]}}} \right. \\
&\quad\left.+ \sin{\term{\frac{b\pi}{a}\term{\mod{x}[a]}}} \cdot \cos{\term{\floor{\frac{x}{a}} b\pi}}\right|.
\end{split}
\end{equation}

Since $\term{\floor{\frac{x}{a}} b\pi} \in \br{0, \pi, 2\pi, 3\pi, \cdots}$, we have that
\[
\sin{\term{\floor{\frac{x}{a}} b\pi}} = 0,
\]
and 
\[
\abs{\cos{\term{\floor{\frac{x}{a}} b\pi}}} = 1.
\]

By combining the previous equalities with~\eqref{eq:sin_prop_1} and~\eqref{eq:sin_prop_2}, Claim~\ref{clm:ax} follows.
\end{proof}

\subsection{Proof of Claim~\ref{clm:bound_Q}}
\begin{proof}
Contradictively assume that $|Q|\geq 1+\log N$, and let $T$ be a subset of $1+\log N$ integers from $Q$. Since $g(q,Q)=1$ for every $q\in Q$, we have
\[
\abs{\br{\mod{q}}}= |Q|\geq 1+\log N.
\]
Observer that (i) the set $T$ has $2^{|T|}> N$ different subsets. Hence it has $O(N^2)$ distinct pair of subsets, and (ii) for any $T'\subset T$ we have that $\sum_{q\in T'}(\mod{q})\in \left[\round{(1+\log N)\frac{N}{2}}\right]$. 
By (i), (ii) and the pigeonhole principle there are two distinct sets $T_1,T_2\subset T$ such that
\[                             
\sum_{q\in T_1}(\mod{q})=\sum_{q\in T_2}(\mod{q}).
\]

%First we prove that there is a partition $(Q_1,Q_2)$ of $Q$ such that
%\begin{equation}\label{q1q2}
%\[
%(\sum_{q\in Q_1}q) \mod N=(\sum_{q\in Q_2}q) \mod N.
%\]
%\end{equation}
%Indeed, by the pigeonhole principle and the assumption $|Q|>\log_2N$, there must be a pair of \emph{distinct} subsets $T_1,T_2$ of $Q$ whose sum $\mod N$ is the same. Hence, $T_1\setminus (T_1\cap T_2)$ and $T_2\setminus (T_1\cap T_2)$ are a pair of \emph{disjoint} subsets of $Q$ whose sum $\mod N$ is the same. By removing the union of these disjoint subsets from $Q$, and continue recursively on the remaining integers in $Q$, we construct a partition of $Q$ into two disjoint subsets $Q_1$ and $Q_2$ such that~\eqref{q1q2} holds.

Put $p\in T_2$, and observe that
\[
\mod{p}=\sum_{q\in T_1}(\mod{q})-\sum_{q\in T_2\setminus\br{p}}(\mod{q}).
\]
Therefore for every $c\in[N]$,
\begin{equation}\label{cpm3}
\begin{split}
&\mod{cp} =\quad\mod{c(\mod{p}))}\\
&=\mod{c\term{\sum_{q\in T_1}(\mod{q})-\sum_{q\in T_2\setminus\br{p}}(\mod{q})}}\\
&\leq \mod{\term{\sum_{q\in T\setminus\br{p}} (\mod{cq})}}.
\end{split}
\end{equation}

Since $g(p,Q)=1$ by the assumption of the claim, there is $c\in C(p)$ such that for every $q \in T$ (where $q \neq p$), either \begin{enumerate*}[label=(\roman*)]
    \item $\mod{cq} \leq \frac{N}{16\log{N}}$,\label{case:1_last} or,
    \item $\mod{cq} \geq \frac{15N}{16\log{N}}$ \label{case:2_last}.
\end{enumerate*}

\noindent\textbf{Handling Case~\ref{case:1_last}.} Assuming that this case holds, then by~\eqref{cpm3} we obtain that
\begin{equation*}
\begin{split}
\mod{cp} &\leq \mod{\term{\sum\limits_{q \in T \setminus \br{p}} \frac{N}{16\log{N}}}} \\
&= \mod{\frac{N}{16}} = \frac{N}{16}.
\end{split}
\end{equation*}
This contradicts the assumption that $c \in C(p)$.

\noindent\textbf{Handling Case~\ref{case:2_last}.} Combining the assumption of this case with~\eqref{cpm3}, yields that
\begin{equation*}
\begin{split}
\mod{cp} &\geq \mod{\term{\sum\limits_{q \in T \setminus \br{p}} \frac{15N}{16\log{N}}}} \\
&= \mod{\frac{15N}{16}} = \frac{15N}{16}.
\end{split}
\end{equation*}
This is a contradiction to the assumption that $c \in C(p)$.
% $$(\mod{cq})< \frac{N}{b\log N}$$ for every $q\in Q\setminus\br{p}$. Thus, using~\eqref{cpm3},
% \[
% (\mod{cp})\leq \sum_{q\in T\setminus\br{p}} \frac{N}{b\log N}\leq \frac{N}{b}.
% \]
% This contradicts the fact that $c\in C(p)$; see definition of $C(p)$ at Lemma~\ref{lem:sensitivity_query_bound}.
\end{proof}

\section{EXTENSION TO HIGH DIMENSIONAL DATA}
\label{sec:highd}
In this section we formally discuss the generalization of our results to constructing coresets for sine fitting of rational high dimensional data. First note that in such (high dimensional) settings, the objective of the \fft problem becomes 
\[
\min_{c \in C} \sum\limits_{p \in P} \sin^2\term{\frac{2\pi}{N} \cdot p^T c},
\]
where $P$ is the set of high dimensional input points and $C$ is the set of queries. Note that, we still assume that both sets are finite and lie on a grid of resolution $\Delta$; see next paragraph for more details. 
%Both the set of input points $P$, and the set of queries $C$ 

%where $C \subseteq \REAL^d$ is the set of queries. 

\paragraph{Assumptions. }To ensure the existence of coresets for the generalized form, we first generalize the assumptions of our results as follows: (i) the original set of queries $[N]$ is now generalized to be the set $B$ of all points with non-negative coordinates and of resolution $\Delta$. Formally speaking, 
let $\Delta >0$ be a rational number that denotes the resolution, and let $X := \br{i\Delta }_{i=0}^{\floor{\frac{N}{\sqrt{d}\Delta}}} \cup \br{N/\sqrt{d}}$ denote the set $\br{0,\Delta ,2\Delta,3\delta,\cdots,N/\sqrt{d}}$, now, our set of queries is defined to be 
\[
B:= X^d =  \underbrace{X \times X \times \cdots \times X}_{d \text{ times}},
\]
i.e., $B\subset \REAL^d$, and for every $i\in [d]$ and $x\in B$, the $i$th coordinate $x_i$ of $x$ is from $X$ (if $x\in B$, then $\forall_{i\in [d]} x_i\in X$). (ii) The input $P$ is contained in the set of queries, thus, our generalization assumes that $P \subseteq B$. 
%which is the set of all points in $\REAL^d$ such that each coordinate of any point $x \in B$ is from $X$, i.e., $B$ has a resolution of $\Delta$.

\paragraph{Sensitivity bound and total sensitivity bound.} First, for every $x\in B$ and $p\in P$ let $D(p,x)= \sin^2\term{\frac{2\pi}{N}p^Tx }$.
%Our results can be easily be generalized to handle points and queries from the $d$-dimensional space, following the following intuitive observation.
The main reason that the our main result which suits the one dimensional setting (where points and queries are integers) is interesting, relies on the fact that each point $p$ in the input set $P$ results a sine wave $Sine(p,\cdot):[N] \to [0,1]$ of a different wavelengths, where specifically for a point $p \in P \subseteq [N]$, the wavelength of the corresponding sine wave is $\frac{2\pi p}{N}$. This ensures that most point don't admit the same squared sine waves which in turn fuels the need to find a small set of points that the sum of their squared sine waves approximate the total sum of the squared sine waves of the input set of integral points in one dimensional space.

Following the same observation, we simply generalize our cost function to account for such traits, where the wavelength of the obtained signals is shown along the direction of the points in $P \subseteq B$; the following figure serves as descriptive illustration.

\begin{figure}[htb!]
    \centering
    \includegraphics{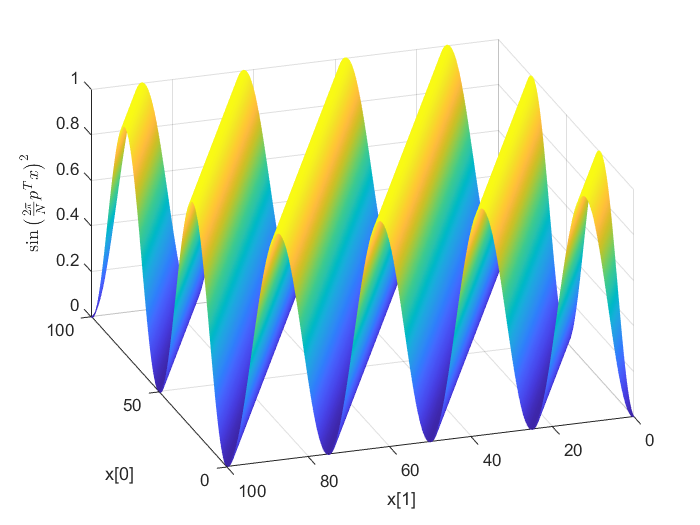}
    \caption{Given a point $p := \begin{pmatrix} 1 \\ 2 \end{pmatrix}$ and a set of queries $B$ where $\Delta := 0.1$ and $N = 100$, the above is a plot of $\sin^2{\term{\frac{2\pi}{N} p^Tx}}$ over every query $x \in B$. Here the $x$-axis denotes the first entry of a query $x \in B$, the $y$-axis denotes the second entry of a query $x$.}
    \label{fig:my_label}
\end{figure}

Following along the ideas above from the one dimensional case, choosing the set of queries to be $B$, and $P$ to be any set of $n$ points contained in $B$, fulfills the same ideas. Thus, we can use $D(p,x) := \sin^2{\term{\frac{2\pi}{N} p^Tx}}$ as our generalized form of squared sine loss function.

Since the dot product is non-negative in our context, it behave as a generalization of the product between two non-negative scalars. Our previous results depends on the product of two scalars rather than the scalar themselves, and from such observation, it can be seen as a leverage point for this generalization to be equipped into our proofs.

\paragraph{Bounding the VC dimension.} It was stated previously that
the necessity of generalizing the assumptions of our results is crucial. Such necessity is needed to handle the case of restricting the VC dimension of the \fft problem in the $d$-dimensional Euclidean space to be finite. The following gives the formal ingredients for such purpose.% and also the bound on the VC dimension of the sine fitting problem in such setting.

\begin{lemma}[Extension of Lemma~\ref{VcbOUND}]
\label{lem:VcbOUND_Ext}
Let $N,n,d$ be a triplet of positive integers where $N \geq n > d$. Let $\Delta>0$ be rational number such that, let $X := \br{i\Delta }_{i=0}^{\floor{\frac{N}{\sqrt{d}\Delta}}} \cup \br{\frac{N}{\sqrt{d}}}$, and let $B := X^d$ denote the set of all $d$-dimensional points such that each coordinate of any point $x \in B$ is from $X$, i.e., $B$ has a resolution of $\Delta$.
Let $P \subseteq B$ be a set of $n$ points. Then the VC dimension of the \fft problem with respect to $P$ and $B$ is $O\term{d \log\term{\frac{N}{\Delta}n}}$. 
\end{lemma}

\begin{proof}
The following proof relies on an intuitive generalization of the proof of Lemma~\ref{VcbOUND}. We note again that the VC-dimension of the set of classifiers that output the sign of a sine wave parametrized by a single parameter (the angular frequency of the sine wave) is infinite. Regardless, our query space is bounded, i.e., every query contained in a ball of radius $N$. Thus, we bound the VC dimension as follows. First let $D(p,x)= \sin^2\term{p^T \cdot x \frac{2\pi}{N}}$ we observe that for every $p \in P$ and $x \in B$, $D(p,x) \leq 1$. Hence, for every $x \in B$ and $r \in [0,\infty)$ it holds that 
$$\br{\RANGES(x,r) \middle| r \geq 0} =  \br{\RANGES(x,r) \middle| r \in [0,1]},$$
where $\RANGES(x,r) = \br{p \in P \mid D(p,x) \leq r}$ is defined as in Definition~\ref{def:dimension}. Secondly, by the definition of $\RANGES$, we have that for every pair of $r_1,r_2 \in [0,1]$ and $x \in B$ where $r_2 \geq r_1$, $\RANGES\term{x,r_2} = \bigcup\limits_{r \in \left[r_1, r_2\right]} \RANGES\term{x,r}.$

This yields that $\abs{\br{\RANGES\term{x,r} \middle| r \in [0,1]}} \leq n$ for any $x \in B$, which consequently means that
$$
\abs{\br{\RANGES\term{x,r} \middle| x \in B, r \geq 0}} \leq n\abs{B},
$$
since $x \in B$ is an integer, and each such $x$ would create a different set of $n$ subsets of $P$. 

Since $B$ is contained in a ball of radius $N$, it holds that $\abs{B} \in O\term{\frac{N^d}{\Delta^d}}$. We thus get that $\forall \, S \subseteq P$,
$$
\abs{\br{S \cap \RANGES\term{x,r} \middle| x \in [N], r \geq 0}} \leq \frac{n}{\Delta^d} \mathrm{vol}\term{B} \in 2^{O\term{d \log\term{\frac{N}{\Delta}n}}}.
$$

The lemma then follows since the above inequality states that the VC dimension is bounded from above by $O\term{d \log\term{\frac{N}{\Delta}n}}$.
\end{proof}

\section{ADDITIONAL EXPERIMENTS}
\label{sec:exp_ext}
In this section, we carry additional experiments to show the advantage of our method in comparison with uniform sampling. We note that Figure~\ref{fig:dog_heart} is given to show that the heartbeat data~\cite{dogHeart} is not of the form of a a flat-like line with regular peaks. In what follows, we show additional experiments on with respect to the sine fitting problem and the data discretization problem. For such task, we consider the following dataset:
\paragraph{Bat Echoes~\citep{singlenueron}} -- acquired from the echolocation pulse emitted by the Large Brown Bat (Eptesicus Fuscus). Such a file has a duration of $2.8ms$ and was digitized by considering a sampling period of $7\mu s$, resulting in a file with $400$ samples.

At Figure~\ref{fig:sine-fit-bat}, it is shown that our coreset clearly outperforms uniform sampling with respect to the \fft problem.
\begin{figure}[t!]
    \centering
    \includegraphics[height=.5\textwidth]{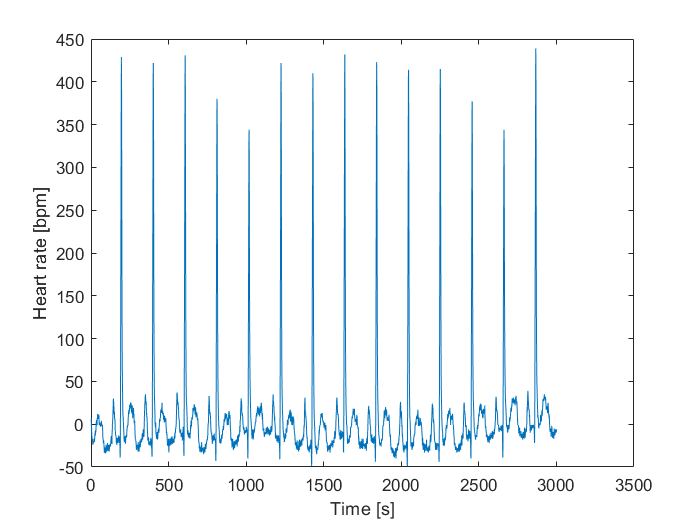}
    \caption{A snippet of the electrocardiogram with respect to a beating heart of some dog.}
    \label{fig:dog_heart}
\end{figure}

\begin{figure*}[htb!]
\includegraphics[width=.49\textwidth]{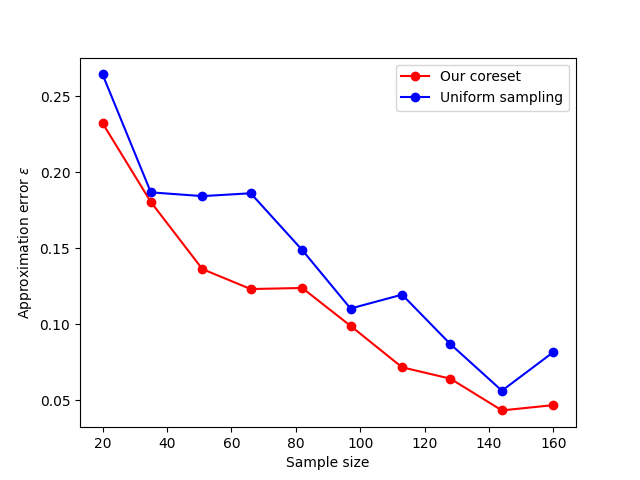}
\includegraphics[width=.49\textwidth]{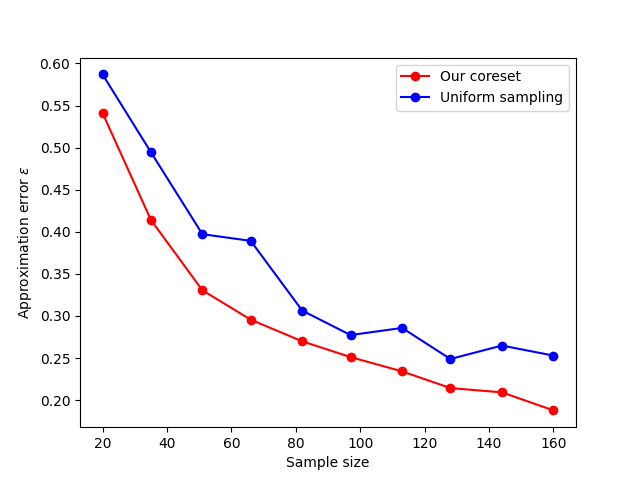}
\caption{The optimal and maximal approximation errors with respect to the \emph{bat echoes} dataset: The x axis is the size of the chosen subset, the y axis is the approximation error across the whole set of queries. The left figure shows the optimal solution approximation error, while the right figure serves to show the maximal approximation error across the whole set of queries.}
\label{fig:sine-fit-bat}
%\end{subfigure}
\end{figure*}

We conclude this section with an experiment done to assess the effectiveness of our coreset against that of uniform sampling for the task of data discretization. Figure~\ref{fig:discrete_results} shows the advantage of using our coreset upon using uniform sampling. For this experiment, the optimal solution of~\eqref{eq:our_cost} with respect to each of the coreset and sampled set of points using uniform sampling is computed. Then using each of the solutions, we generate a sine waves such that its phase is equal to the computed solutions. We then project the points of the whole input data on the closest roots of these waves respectively, resulting into two sets of points. Now, we compute the distance between each point and its projected point and sum up the distance for each of the projected sets of points. Finally, we compute the approximated yielded by these values with respect to the value which we would have obtained if this process was done solely on the whole data. 
\newpage
\begin{figure}[htb!]
    \centering
    \includegraphics[height=.5\textwidth]{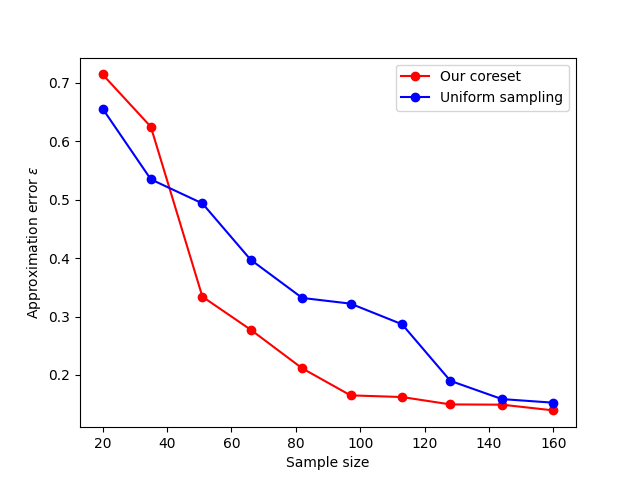}
    \caption{The optimal approximation error with respect to the \emph{bat echoes} dataset: The x axis is the size of the chosen subset, the y axis is the approximation error across the whole set of queries. The left figure shows the optimal solution approximation error, while the right figure serves to show the maximal approximation error across the whole set of queries. This figure is with respect to the data discretization problem.}
    \label{fig:discrete_results}
\end{figure}

\end{document}